\documentclass[12pt]{article}

\usepackage{amsmath,amssymb,amsthm,mathtools,bm}
\usepackage{setspace,xcolor,enumitem,url,booktabs,soul,multirow,arydshln,framed,hyperref}
\usepackage{times}
\usepackage{natbib}
\usepackage[plain,noend]{algorithm2e}

\usepackage[T1]{fontenc}
\usepackage{tgbonum} \fontfamily{qbk} 

\newtheorem{theorem}{Theorem}

\usepackage{authblk}
\author[1,2]{Li-Chun Zhang}
\author[3]{Danhyang Lee}
\affil[1]{\em \small Statistisk sentralbyrå (lcz@ssb.no)}
\affil[2]{\em \small University of Southampton (L.Zhang@soton.ac.uk)}
\affil[2]{\em \small University of Alabama (dlee84@cba.ua.edu)}
\title{Design-based individual prediction}
\date{}

\begin{document}

\maketitle

\begin{abstract} \noindent
A design-based individual prediction approach is developed based on the expected cross-validation results, given the sampling design and the sample-splitting design for cross-validation. Whether the predictor is selected from an ensemble of models or a weighted average of them, valid inference of the unobserved prediction errors is defined and obtained with respect to the sampling design, while outcomes and features are treated as constants. 
\end{abstract}


\section{Introduction}

Valid inference of the unobserved individual prediction errors is a fundamental issue to supervised machine learning, no matter how confident one is about the obtained predictor. An IID model of the prediction errors is commonly assumed for algorithm-based learning, such as random forest, support vector machine or neural network, which could be misleading in situations where the available observations are not obtained in a completely random fashion. 

We define and develop a design-based approach to individual prediction, which requires the sample for learning to be selected by a  probability design. Whether the adopted predictor is selected from an ensemble of models or a weighted average of them, the proposed approach can provide valid inference of the associated risk with respect to the known sampling design, “irrespectively of the unknown properties of the target population studied” \citep{neyman1934two}. 

For an illustration of the conceptual issue at hand, suppose on observing an outcome value $y$ in a subset $s$ of a finite collection of units $U$, denoted by $\{ y_i : i\in s\}$ and $s\subset U$, one would like to predict the $y$-value for each unit out of $s$ by $\mu(s) = \sum_{i\in s} y_i/n$, where $n$ is the number of units in $s$. How can one infer about the loss $D_s =\sum_{j\in U\setminus s} \{\mu(s) - y_j \}^2$ that is unobserved?

One possibility is to evaluate $D_s$ using an assumed model. Under the model that $y_i$ is independent and identically distributed (IID), for any $i\in U$, we have
\[
E_M(D_s \mid s) = (N-n) \big( 1 + n^{-1} \big) \sigma^2
\]
where $E_M(\cdot |s)$ denotes expectation with respect to the IID model conditional on the given subset $s$, and $\sigma^2$ is the variance of $y_i$ under the model, and $N$ is the number of units in $U$.

A fundamentally different, design-based approach would be possible if $s$ is selected from $U$ by a known sampling design, denoted by $s\sim p(s)$, where $\sum_{s\in \Omega} p(s) =1$ and $\Omega$ contains all the possible samples from $U$. For instance, we have $p(s) = 1/\binom{N}{n}$ if $s$ is selected from $U$ by simple random sampling without replacement. We have then
\[
E_p(D_s) = (N-n) \big( 1 + n^{-1} \big) S_y^2
\]
where $E_p$ denotes expectation with respect to $p(s)$, while $y_U = \{ y_i : i\in U\}$ are treated as constants and $S_y^2 = \sum_{i\in U} (y_i - \bar{Y})^2/(N-1)$ and $\bar{Y} = \sum_{i\in U} y_i/N$.   

Since $s_y^2 = \sum_{i\in s} \{y_i - \mu(s) \}^2/(n-1)$ is both unbiased for $\sigma^2$ under the IID model {and} unbiased for $S_y^2$ under simple random sampling, numerically one would obtain the same estimate of the expected loss, although the two measures have completely different interpretations. While some assumed model would be necessary for the measure $E_M(D_s |s)$ if the selection mechanism of $s$ is unknown, it could be invalid if the observed data distribution actually differs to that of the unobserved ones. While the design-based measure $E_p(D_s)$ requires one to plan and implement the simple random sampling design here, it is necessarily valid because the sampling distribution $p(s)$ is known. 

The general design-based individual prediction approach we develop in this paper can use any model or algorithm-based predictor learned from the observed sample. The prediction errors over the out-of-sample units are evaluated with respect to the known sampling design, while all the outcomes and relevant features in the population are treated as constants. We note that \cite{sande2021design} study design-unbiased learning for estimating population totals such as $Y = \sum_{i\in U} y_i$ in the same spirit. 

By the design-based approach, there is no need to assume that a true model exists for $y_U$, or that one is able to identify and learn the true predictor under repeated sampling. It is then natural to combine an ensemble of different predictors (e.g. \citealp{dietterich2000ensemble, zhou2012ensemble, sagi2018ensemble, dong2020survey}) in addition to selecting a single best predictor. Design-based ensemble learning by voting or averaging will be developed.

\section{Theory} \label{sec:theory}

In addition to $y_U$, denote by $x_U = \{ x_i : i\in U\}$ the vectors of features. Given any sample $s$, let $\mu(x, s)$ be the predictor for the out-of-sample units, $R = U\setminus s$. Regardless how $\mu(x, s)$ is obtained from $y_s = \{ y_i : i\in s\}$ and $x_s = \{ x_i : i\in s\}$, its unobserved \emph{total squared error of prediction (TSEP)} is given by
\[
D(s; \mu) = \sum_{i\in R} \left\{ \mu(x_i, s) - y_i \right\}^2 ~.
\]
We define the \emph{risk} of $\mu$ to be the expectation of $D(s; \mu)$ over repeated sampling of $s\sim p(s)$, denoted by
\begin{equation} \label{risk}
\tau(\mu) = E_p\left\{ D(s; \mu) \right\}.
\end{equation} 
Design-based prediction is guided by the risk \eqref{risk} over the sampling distribution $p(s)$. First, valid estimation of $\tau$ is developed below, where $s$ is random but $y_U$ and $x_U$ are treated as constants associated with the given $U$. Then, design-based approaches to ensemble learning by voting and averaging are described.

\subsection{Sample split and associated TSEP estimation}

Denote by $s_1 \cup s_2 = s$ and $s_1 \cap s_2 = \emptyset$ a \emph{training-test sample split}, where $s_1$ is selected by a sample-splitting design, denoted by $q(s_1 | s)$. For instance, $s_1$ of a given size can be sampled from $s$ with or without replacement. In $T$-fold cross-validation, $s$ is first randomly divided into $T$ clusters and then each cluster is selected as $s_2$ one by one systematically, yielding $s_1 = s\setminus s_2$ accordingly. 

Given any sample-split  $(s_1, s_2)$, and any predictor $\mu(x, s_1)$ that is trained on $s_1$, let its prediction error for any $i\notin s_1$ be 
\[
e_i(\mu, s_1) = \mu(x_i, s_1) - y_i 
\]
which can be observed for any $i \in s_2$ but not for any $i\in R$. However, the TSEP \
\[
D_R(s_1; \mu) = \sum_{i\in R} e_i(\mu, s_1)^2
\]
on applying $\mu(x, s_1)$ to $R = U\setminus s$ can be estimated unbiasedly as follows. 

Given the training set $s_1$, for any $i\in U\setminus s_1$, let 
\begin{equation} \label{pi2}
\pi_{2i} = \Pr(i\in s_2 \mid s_1)  = \sum_{s: i\in s \cap i\notin s_1} f(s\mid s_1)
\end{equation} 
be the conditional inclusion probability in $s_2$ given $s_1$, which is derived from the given $pq$-design defined by $p(s)$ and $q(s_1 \mid s)$, i.e.  
\[
f(s_1, s) = q(s_1 \mid s) p(s) = f(s \mid s_1) f(s_1) ~.
\] 
The estimator 
\[
\hat{D}_R(s_1; \mu) =  \sum_{i\in s_2} \big( \pi_{2i}^{-1} - 1 \big) e_i(\mu, s_1)^2 
\]
is unbiased for $D_R(s_1; \mu)$ conditional on $s_1$ in the sense that
\begin{equation} \label{subsample}
E_s\{ \hat{D}_R(s_1; \mu) \mid s_1\} = E_s\{ D_R(s_1; \mu) \mid s_1 \}~.
\end{equation}
To derive \eqref{subsample}, let $A_2 = \sum_{i\in s_2} e_i(\mu, s_1)^2 = \sum_{i\in U\setminus s_1} e_i(\mu, s_1)^2 - D_R(s_1; \mu)$. Both $A_2$ and $D_R(s_1; \mu)$ vary with $s_2 = s\setminus s_1$ but their sum $\sum_{i\in U\setminus s_1} e_i(\mu, s_1)^2$ is fixed given $s_1$. Thus, conditionally over all possible $s$ given $s_1$, we have
\[
E_s\{ \hat{D}_R(s_1; \mu) \mid s_1\} = E_s\{ \sum_{i\in s_2} \pi_{2i}^{-1} e_i(\mu, s_1)^2 - A_2 \mid s_1 \} =  E_s\{ D_R(s_1; \mu) \mid s_1 \}~.
\]

\subsection{Subsampling Rao-Blackwellisation and risk}

One can only observe the residuals of any $\mu(x, s)$ that is trained directly on the full sampl $s$. Whereas for unbiased estimation of the risk \eqref{risk}, we would use prediction errors generally, in order to capitalise on the result \eqref{subsample}. We shall therefore consider the predictor obtained over all possible sample-splits, i.e. $s_1 \sim q(s_1 \mid s)$, which is defined as 
\begin{equation} \label{SRB}
\bar{\mu}(x, s) = E_q\{\mu(x, s_1) \mid s\} ~.
\end{equation}
Due to the Rao-Blackwell Theorem (\citealp{rao1945info, blackwell1947conditional}), we refer to the operation $E_q(\cdot \mid s)$ as \emph{subsampling Rao-Blackwellisation (SRB)}, since it is the conditional expectation over the subsampling distribution $q(s_1 \mid s)$, where the unordered $s$ is the minimal sufficient statistic with respect to $p(s)$.

We refer to $\bar{\mu}(x, s)$ given by \eqref{SRB} as the {SRB-predictor}, associated with 
\[
D(s; \bar{\mu}) = \sum_{i\in R} e_i(\bar{\mu})^2 \qquad\text{and}\qquad e_i(\bar{\mu}) = \bar{\mu}(x_i, s) - y_i ~.
\] 
For any $i\in R$ with $x_i =x$, the errors of $\bar{\mu}(x, s)$ and $\mu(x, s_1)$ are related by 
\[
e_i(\mu, s_1) = \mu(x, s_1) - y_i = \left\{ \mu(x, s_1) - \bar{\mu}(x, s) \right\} + e_i(\bar{\mu}) ~.
\]
Since $e_i(\bar{\mu})$ is constant of $s_1$ and $\bar{\mu}(x, s) = E_q\{ \mu(x, s_1) | s \}$ by definition \eqref{SRB}, we have
\begin{equation} \label{e2SRB}
e_i(\bar{\mu})^2 = E_q\{ e_i(\mu, s_1)^2 \mid s \} - E_q\{ a_i(\mu, s_1)^2 \mid s \} 
\end{equation}
where $a_i(\mu, s_1) = \mu(x, s_1) - \bar{\mu}(x, s)$ and $E_q\{ a_i(\mu, s_1)^2 \mid s \}$ is the variance of $\mu(x, s_1)$ under SRB.

\begin{theorem} For any given $\mu(\cdot)$, an unbiased estimator of the risk $\tau(\bar{\mu})$ of the corresponding SRB-predictor, over $s\sim p(s)$, is given by
\[
\hat{D}(s; \bar{\mu}) = E_q\Big( \sum_{i \in s_2} ( \pi_{2i}^{-1} -1 ) \left\{ e_i(\mu, s_1)^2 - a_i(\mu, s_1)^2 \right\} \mid s \Big)~.
\] 
\label{theorem}
\end{theorem}

\begin{proof} By \eqref{e2SRB}, we have 
\[
D(s; \bar{\mu}) = E_q\{ \sum_{i\in R} e_i(\mu, s_1)^2 \mid s \} - E_q\{ \sum_{i\in R} a_i(\mu, s_1)^2 \mid s \} 
\]
For the first term that can be rewritten as $E_q\{ D_R(s_1; \mu) \mid s \}$, the estimator 
\begin{equation} \label{ARhat}
\hat{E}_q\big( D_R(s_1; \mu) \mid s \big) = E_q\big(\hat{D}_R(s_1; \mu) \mid s\big)
\end{equation}
is unbiased over $p(s)$ since, using \eqref{subsample}, we have
\begin{align*}
E_p\{ E_q\big( \hat{D}_R(s_1; \mu) \mid s\big) \} & = E_{s_1}\{ E_s\big( \hat{D}_R(s_1; \mu)\mid s_1\big) \}  \\
& = E_{s_1}\{ E_s\big( D_R(s_1; \mu) \mid s_1\big) \} = E_p\left\{E_q\big( D_R(s_1; \mu) \mid s \big) \right\}.
\end{align*}
On replacing $e_i(\mu, s_1)^2$ by $a_i(\mu, s_1)^2$, one can carry through the same derivation for the second term of $D(s; \bar{\mu})$ above, $E_q\{ \sum_{i\in R} a_i(\mu, s_1)^2 \mid s \}$. It follows that
\[
E_p\{\hat{D}(s;\bar{\mu}) \} = E_p\{ D(s;\bar{\mu}) \} = \tau(\bar{\mu}) ~.
\]
This completes the proof.
\end{proof}

\subsection{Notes on implementation}

Exact SRB operation is not tractable analytically for algorithm-based $\mu(\cdot)$, such as random forest or support vector machine. Nor is exact SRB numerically feasible if the number of possible $s_1$ given $s$ is large. In practice, one can use the Monte Carlo SRB predictor based on $T$ subsamples, which is given as
\[
\begin{cases} \tilde{\mu}(x_i, s) = T^{-1} \sum_{t=1}^T \mu(x_i, s_1^{(t)}) & \text{if } i\in R \\
	\mathring{\mu}(x_i, s) = T_i^{-1} \sum_{t=1}^T \mathbb{I}\big( i\notin s_1^{(t)} \big) \mu(x_i, s_1^{(t)}) & \text{if } i\in s
\end{cases}
\]
where $s_1^{(t)}$ is the $t$-th subsample by $q(s_1 \mid s)$ and $T_i =  \sum_{t=1}^T \mathbb{I}\big( i\notin s_1^{(t)} \big)$. 

For the estimation of risk, let $e_i(\mu, s_1^{(t)}) = \mu(x_i, s_1^{(t)}) - y_i$ for any $i\in s_2^{(t)}$ directly, whereas one can use $a_i(\mu, s_1^{(t)}) = \mu(x_i, s_1^{(t)}) - \mathring{\mu}(x_i, s)$ as an out-of-bag approximation to $\mu(x_i, s_1^{(t)}) - \bar{\mu}(x_i, s)$, for any $i\in s_2^{(t)}$, instead of $\mu(x_i, s_1^{(t)}) - \tilde{\mu}(x_i, s)$ that would be a residual-based alternative. The Monte Carlo risk estimator is given by
\begin{equation} \label{D.est}
\tilde{D}(s; \bar{\mu}) = \frac{1}{T} \sum_{t=1}^T \sum_{i\in s_2^{(t)}} \big( \pi_{2i}^{-1} -1 \big) \{ e_i(\mu, s_1^{(t)})^2 - a_i(\mu, s_1^{(t)})^2 \}~. 
\end{equation}

By definition, $\pi_{2i}$ in \eqref{D.est} is the conditional $s_2$-inclusion probability given $s_1$, which is given by \eqref{pi2} and requires $f(s\mid s_1)$ that is derived from $q(s_1 \mid s) p(s)$. However, $p(s)$ is unknown for many unequal-probability sampling algorithm in practice, e.g. the cube method \citep{deville2004cube}, although the inclusion probability $\pi_i = \Pr(i\in s)$ is always known. 

One can use instead another sampling probability of the $pq$-design. For any $i\in U$, let its conditional test-set inclusion probability \emph{given $i\notin s_1$} be
\begin{equation} \label{phi2}
\phi_{2i} = \Pr(i\in s_2 \mid i\notin s_1) = \frac{\Pr(i\in s_2, i\notin s_1)}{\Pr(i\notin s_1)} 
= \frac{\pi_i \{ 1 - \Pr(i\in s_1 \mid i\in s)\}}{1 -  \pi_i \Pr(i\in s_1 \mid i\in s)} ~.
\end{equation}
Given $\pi_i$, the probability $\phi_{2i}$ can be calculated as long as $\Pr(i\in s_1 \mid i\in s)$ does not depend on $i$ and can be specified regardless the realised $s$, such as simple random sampling (SRS) of $s_1$ from $s$ with or without replacement, or $T$-fold cross-validation. Since
\[
\phi_{2i}  = \frac{\sum_{s_1:i\notin s_1} \sum_{s: i\in s \cap i\notin s_1} f(s\mid s_1) f(s_1)}{\sum_{s_1:i\notin s_1} f(s_1)}  
= \frac{\sum_{s_1:i\notin s_1} \pi_{2i} f(s_1)}{\sum_{s_1:i\notin s_1} f(s_1)}  
= E_{s_1}\{ \pi_{2i} \mid i\notin s_1\} ~,
\]
it is the conditional expectation of non-zero $\pi_{2i}$, since $\pi_{2i} =0$ iff $i\in s_1$.

For instance, take the special case of SRS without replacement of $s$ from $U$ and SRS without replacement of $s_1$ from $s$, with sample sizes $n = |s|$, $n_1 = |s_1|$ and $n_2 = |s_2| = n - n_1$. For any given $s_1$ and $i\notin s_1$, we have exactly
\begin{align*}
\pi_{2i} = \frac{\Pr(i\in s_2) f(s_1 \mid i\in s_2)}{f(s_1)} = \frac{\frac{n_2}{N} \binom{N-1}{n_1}^{-1}}{\binom{N}{n_1}^{-1}} 
= \frac{n_2}{N-n_1} = \phi_{2i}
\end{align*}

\subsection{SRB-selector}

Beyond using a single base learner $\mu$ for the SRB-predictor \eqref{SRB}, consider design-based ensemble learning by voting given an order-$K$ heterogeneous ensemble $\{ \mu_1, ..., \mu_K \}$. Let $D(s; \mu_k) = \sum_{i\in R} \{ \mu_k(x_i, s) - y_i \}^2$. Denote by $\Omega = \bigcup_{k=1}^K \Omega_k$ the partition of the sample space such that, for any $s\in \Omega_k$ and $l\neq k$, we have
\[
D(s; \mu_k) < D(s; \mu_l)
\]
where we discount the possibility of $D(s; \mu_k) = D(s; \mu_l)$ merely to simplify the exposition. To {select} a single predictor for $R$ based on a given sample $s$, which minimises the risk \eqref{risk}, one would vote for $\mu_k(x, s)$ iff $s\in \Omega_k$. It follows that, for design-based ensemble learning by voting, the optimal selector is the perfect classifier of $\mathbb{I}(s\in \Omega_k)$.

It is a common approach to select a model by cross-validation and majority-vote, where cross-validation is based on $s_1 \sim q(s_1 | s)$ and $s_2 = s\setminus s_1$. The expected selection result is given by the {SRB-selector} below. Given any $(s_1, s_2)$ by $q(s_1 | s)$ and any $k=1, ..., K$, let
\[
\delta_k(s_1) = \begin{cases} 1 & \text{if}~~ k = \arg \min\limits_{l=1,...,K}~ \sum\limits_{i\in s_2} \left\{ \mu_l(x_i, s_1) - y_i\right\}^2\\ 
	0 & \text{otherwise} \end{cases}  
\]
indicate which predictor has the least sum of squared errors in $s_2$. The {SRB-selector} 
\begin{equation} \label{SRB-selector}
\bar{\delta}_k(s) = \begin{cases} 1 & \text{if}~~ k = \arg \max\limits_{l=1,...,K}~E_q\left\{ \delta_k(s_1) | s \right\} \\ 
		0 & \text{otherwise} \end{cases}  
\end{equation}
is a classifier of $\mathbb{I}(s\in \Omega_k)$, i.e. the expected majority-vote over cross-validation. 

Given the selection by \eqref{SRB-selector}, say, $\mu_k$, one can reuse the same cross-validation samples $(s_1, s_2)$ to obtain the selected SRB-predictor $\bar{\mu}_k$ and its associated risk.

\subsection{Mixed SRB-predictor}

For design-based averaging given an order-$K$ ensemble $\{ \mu_1, ..., \mu_K \}$, let the \emph{mixed SRB-predictor} be
\begin{equation} \label{mixSRB-K}
\mu(x, s) = \sum_{k=1}^K w_k \bar{\mu}_k(x, s)
\end{equation} 
where $\sum_{k=1}^K w_k = 1$ and $w_k > 0$ for the mixing weights, $k=1, \ldots, K$. We have
\begin{align*}
D(s; \mu) = \sum_{k\neq 1} w_k^2 D_{kk} + \big( 1- \sum_{l\neq 1} w_l\big)^2 D_{11} 
	+ \sum_{k=1}^K \sum_{l\neq k,1} w_k w_l D_{kl} + 2 \sum_{k\neq 1} w_k (1- \sum_{l\neq 1} w_l) D_{1k} 
\end{align*}
now that $w_1 = 1 - \sum_{k\neq 1} w_k$, where $D_{kk} = D(s; \bar{\mu}_k)$ and $D_{kl} = D(s; \bar{\mu}_k, \bar{\mu}_l)$ is given by 
\[
D_{kl} = \sum_{i \in R} e_i(\bar{\mu}_1) e_i(\bar{\mu}_2) 
= \sum_{i \in R} E_q\{e_i(\mu_1, s_1)e_i(\mu_2, s_1) \mid s\} - E_q\{a_i(\mu_1, s_2) a_i(\mu_2, s_2) \mid s\}, 
\]
i.e. similarly to \eqref{e2SRB}. An estimator of $D_{kl}$ follows as a corollary of Theorem \ref{theorem}, as well as its Monte Carlo approximation similarly to \eqref{D.est}.

The optimal mixing weights $w_k$ minimise $D(s; \mu)$. The estimated $\hat{w}_k$ can be obtained via $\hat{D}(s; \mu)$ given $\hat{D}_{kl}$, for all $k,l = 1, ..., K$. Substituting $\hat{w}_k$ in \eqref{mixSRB-K} yields the mixed SRB-predictor. The associated risk \eqref{risk} can be estimated by $\hat{D}(s; \mu)$. 

Whilst the above approach aims at minimum risk \eqref{risk}, it may experience instability when the ensemble is not sufficiently heterogeneous. A robust approach to mixed ensemble prediction should automatically aim at the same mixing weight of two component predictors that are equal to other. 

For any $k=1, ..., K$, write $\hat{D}(s; \bar{\mu}_k) = E_q( \hat{\tau}_k \mid s )$, similarly to $\hat{D}(s; \bar{\mu})$ in Theorem \ref{theorem}. Regarding the risk of $\bar{\mu}_k(x, s)$ defined by \eqref{risk}, we have
\[
\tau(\bar{\mu}_k) = E_{s_1}\{\tau(\bar{\mu}_k | s_1) \} = E_{s_1}\left\{ E_s\big( \hat{\tau}_k \mid s_1 \big)\right\} 
= E_p\left\{ E_q\big( \hat{\tau}_k \mid s \big) \right\}  
\]
where $\tau(\bar{\mu}_k | s_1)$ is its conditional risk given $s_1$. Let the SRB operation yield
\begin{equation} \label{alpha}
w_k = E_q(\delta_k \mid s),  \quad
	\delta_k = \begin{cases} 1 & \text{if}~~ k = \arg \min\limits_{l=1,...,K} \hat{\tau}_l \\ 0  & \text{otherwise} \end{cases} .
\end{equation}
The corresponding mixed SRB-predictor \eqref{mixSRB-K} is robust against $\mu_k \approx \mu_l$ for any $k\neq l$. While the SRB-selector \eqref{SRB-selector} is a binary classifier taking the majority-vote over all $(s_1, s_2)$, the robust mixing weight \eqref{alpha} is a proportion over all the votes.

\section{Illustration} \label{sec:simulation}

Simulations below provide a simple illustration of the design-based individual prediction approach and the potential pitfalls of the IID error model.

We generate $B=200$ sets of $y_U$ of population size $N=2000$ in an {ad hoc} manner. For each $y_U$, half of them are generated by M1 below and half of them by M2, where $x_1 \sim N(0,1)$ and $x_2 \sim \mbox{Poisson}(5)$, 
\begin{align*}
\mbox{(M1)}\quad y & = x_1 + 0.5 x_2 + \epsilon, \quad 
	\epsilon\sim \begin{cases} N(0,1) & \mbox{if } z=1 \Leftrightarrow x_2 < 3\\ 
		N(-2,1) & \mbox{if } z=2 \Leftrightarrow 3\leq x_2 < 7\\ N(2,1) & \mbox{if } z= 3\Leftrightarrow x_2 \geq 7 \end{cases} \\
\mbox{(M2)}\quad y & = 0.5 + 1.5x_1 + x_2 + \epsilon, \quad \epsilon \sim z^2 + N(0, 0.25), \quad z\sim N(0,1),
\end{align*}
From each population we draw a sample of size $n=200$ either by SRS without replacement or Poisson sampling. For Poisson sampling, we set $\pi_i^{-1} \propto 1+ 1/\mathrm{exp}(\alpha+0.5y_i)$ and $\sum_{i \in U} \pi_i =n$, where $\alpha \in \{ 1, -0.1, -1\}$ leads to the coefficient of variation of $\pi_i$ over $U$, denoted by cv$_{\pi}$, to be about 15\%, 30\% and 45\%, respectively. This illustrates a situation where sample selection may cause issues for model-based uncertainty assessment.  

Let an order-3 ensemble contain linear regression, random forest and support vector machine. Let the feature vector be $x = (x_1, x_2)$ in all cases. We use a 70-30 random split for subsampling of $(s_1, s_2)$ and $T = 50$ for relevant Monte Carlo SRB operations such as \eqref{D.est}. We obtain the SRB-predictor \eqref{SRB} selected by \eqref{SRB-selector}, and the two mixed SRB-predictors using weights that are either optimal for \eqref{mixSRB-K} or robust \eqref{alpha}. Moreover, for each simulation $b=1, ..., B$, let the hypothetical SRB-selector \eqref{SRB-selector} be
\[
\bar{\delta}_k(s^{(b)}) = 
\begin{cases} 1 & \text{if}~ k = \arg \min\limits_{l=1,...,K} \sum\limits_{i\in R^{(b)}} \left\{\tilde{\mu}_l(x_i, s^{(b)}) - y_i\right\}^2 \\ 
	0 & \text{otherwise} \end{cases}  
\]
which is based on the true total squared error of predictions of the Monte Carlo SRB-predictors; let the hypothetical optimal mixing weights $w_k^{(b)}$ minimise the true total squared error of prediction of \eqref{mixSRB-K}. This yields the hypothetical perfectly selected or optimally mixed predictors, respectively.

For each predictor, we estimate its standardised risk \eqref{risk}, $\tau/|R|$, as described in Section \ref{sec:theory}, where we have $\pi_{2i} \equiv n_2/(N- n_1)$ under SRS given $n_1 = |s_1|$ and $n_2 = |s_2|$, and we use $\phi_{2i}$ given by \eqref{phi2} instead of $\pi_{2i}$ under Poisson sampling. Note that if $\hat{\tau}$ is unbiased for $\tau$ over repeated sampling from a given population, then it is also unbiased for the average of mean squared error of prediction, $D/|R|$, over all the 200 populations. Note also that $\hat{\tau}$ calculated for the two hypothetical predictors are not affected by the actual uncertainty associated with model selection or estimating the mixing weights. 

For comparison, we consider two estimators of the mean squared error of prediction,  both of which rely on the IID model of prediction errors; see e.g. \cite{james2013introduction}. 
\begin{itemize}[leftmargin=6mm]
\item {Residual-based} estimator $\sum_{i\in s} \hat{e}_i^2/n$ where $\hat{e}_i = \tilde{\mu}(x_i, s)- y_i$ for given predictor. 
\item Given $\mu$ either selected or mixed, let the cross-validation-based estimator be
\[
\frac{1}{T} \sum_{t=1}^T \frac{1}{n_2} \sum_{i\in s_2^{(t)}} \{\mu(x_i, s_1^{(t)}) - y_i\}^2, 
\quad \tilde{\mu}(x_i, s) = \frac{1}{T} \sum_{t=1}^T \mu(x_i, s_1^{(t)}) ~.
\]
\end{itemize}

\begin{table}[ht]
\centering
\caption{Proportion selected by \eqref{SRB-selector}, average optimal \eqref{mixSRB-K} or robust \eqref{alpha} mixing weights. SRS, without replacement; PS, Poisson sampling; LR, Linear regression; RF, Random forest; SVM, Support vector machine.}
\begin{tabular}{cccccccc} \toprule
& & \multicolumn{3}{c}{SRS} & \multicolumn{3}{c}{PS (cv$_{\pi}$=15\%)}\\ 
& & LR & RF & SVM & LR & RF & SVM \\ \cline{3-8}
\multirow{2}{*}{Hypothetical}& Selected  & 1  & 0 & 0 & 1 & 0 & 0 \\ 
& Mixed, optimal & 0.73 & 0.21 & 0.06 & 0.81 & 0.18 & 0.01 \\ \cline{2-8}
\multirow{3}{*}{Actual}  & Selected & 0.97  & 0.02 & 0.01 & 1 & 0 & 0 \\ 
&		Mixed, optimal & 0.74 & 0.18 & 0.08 & 0.82 & 0.15 & 0.02 \\
&	Mixed, robust & 0.68 & 0.20 & 0.12 & 0.75 & 0.18 & 0.07 \\ \toprule
&	& \multicolumn{3}{c}{PS (cv$_{\pi}$=30\%)} & \multicolumn{3}{c}{PS (cv$_{\pi}$=45\%)}\\ 
&	& LR & RF & SVM & LR & RF & SVM \\ \cline{3-8}
\multirow{2}{*}{Hypothetical}	&	Selected & 1 & 0 & 0 & 1 & 0 & 0 \\ 
&	Mixed, optimal & 0.9 & 0.1 & 0 &  0.98 & 0.02 & 0 \\ \cline{2-8}
 \multirow{3}{*}{Actual} 	&	Selected & 0.995 & 0 & 0.005 & 1 & 0 & 0 \\ 
&	Mixed, optimal & 0.87 & 0.13 & 0 & 0.84 & 0.08 & 0 \\
&	Mixed, robust & 0.78 & 0.19 & 0.03 & 0.80 & 0.18 & 0.02 \\ \bottomrule
\end{tabular} \label{tab:selection}
\end{table}

Table \ref{tab:selection} provides a summary of the models and mixing weights used over the 200 simulations. Linear regression is the best single-model predictor for all the $(y_U, x_U, s)$ generated by the ad hoc mechanism. The hypothetical and actual selectors (or optimal mixing weights) are quite close to each other, whereas the linear regression model is somewhat weighted down for the robust mixing weights. In any case, we are not concerned with best possible prediction here, but rather valid estimation of the errors of any given predictor. 

\begin{table}[ht]
\centering
\caption{Mean squared error of prediction ($D/|R|$) and estimates, averaged over 200 simulations; predictor selected by majority-vote, or averaged by optimal or robust mixing weights. MSEP, Mean squared error of prediction; SRS, without replacement; PS, Poisson Sampling; CV-based, Cross-validation-based.} 
\begin{tabular}{ccccccc} \toprule
& \multicolumn{3}{c}{SRS} & \multicolumn{3}{c}{PS (cv$_{\pi}$=15\%)} \\ 
MSEP $D/|R|$ & Selected & Optimal & Robust & Selected & Optimal & Robust \\ 
Average, true  & 8.432 & 8.367 & 8.380 & 8.566 & 8.558 & 8.578 \\ 
Design, hypothetical & 8.399 & 8.300 & - & 8.566 & 8.513 & - \\ 
Design, actual & 8.395 & 8.260 & 8.284 & 8.416 & 8.343 & 8.372 \\ 
Model, CV-based  & 8.453 & 8.341 & 8.374 & 8.014 & 7.981 & 8.009 \\ 
Model, residual-based   & 8.076 & 7.264 & 7.178 & 7.766 & 7.146 & 7.008 \\  \cline{2-7}
& \multicolumn{3}{c}{PS (cv$_{\pi}$=30\%)} & \multicolumn{3}{c}{PS (cv$_{\pi}$=45\%)} \\ 
MSEP $D/|R|$ & Selected & Optimal & Robust & Selected & Optimal & Robust \\ 
Average, true & 9.021 & 9.043 & 9.072 & 9.866 & 9.915 & 9.981 \\ 
Design, hypothetical & 9.013 & 8.994 &  -  & 9.866 & 9.862 & - \\ 
Design, actual  & 8.767 & 8.711 & 8.747 & 9.288 & 9.257 & 9.316 \\ 
Model, CV-based & 7.566 & 7.563 & 7.590  & 6.992 & 6.997 & 7.035 \\ 
Model, residual-based  & 7.323 & 6.861 & 6.638& 6.776 & 6.509 & 6.187 \\ \bottomrule
\end{tabular} \label{tab:inference}
\end{table}

Table \ref{tab:inference} shows the average of mean squared error of prediction and its estimates over the 200 simulations. The design-based risk estimator for the two hypothetical predictors are essentially unbiased; the approximate $\phi_{2i}$ under Poisson sampling have worked as well as the exact $\pi_{2i}$ under SRS. Due to the missing uncertainty associated with the actual SRB-selector and mixing-weights, the risk estimator for the three actual predictors exhibit some underestimation, the worst of which is $-6.6\%$ under Poisson sampling with cv$_{\pi}$=45\%.  

The cross-validation-based estimator of mean squared error of prediction is nearly unbiased when the sample is actually selected by SRS, where the out-of-bag squared prediction errors in the test sample $s_2$ have the same mean as those in $R$, conditional on $s_1$ under the twice-SRS $pq$-design. This is reasonable, because the IID error model would be valid under SRS with replacement. However,  the cross-validation-based estimator can become severely biased, if the IID model does not hold for the actual sample selection mechanism, as illustrated here for Poisson sampling as cv$_{\pi}$ increases. Finally, residual-based estimation of mean squared error of prediction should be avoided because it causes underestimation generally, e.g. the bias is severe for the mixed predictors even under simple random sampling.

\section{Final remarks} \label{sec:final}

Although the IID model of the unobserved individual prediction errors is commonly applied  for algorithm-based machine learning, it can be misleading in situations where the observations are not selected with an equal probability.

We define and develop a design-based approach to individual prediction, which requires the sample for learning to be selected by a  probability design. Ensemble learning by voting or averaging is aimed at the expected predictor by cross-validation. Valid inference of the associated risk \eqref{risk} can be obtained with respect to the known probability design, regardless if the adopted predictor is true or not. In practice, stratified simple random sampling can be considered instead of individually varying sampling probabilities, which can help to reduce the uncertainty due to model-selection or mixing-weight estimation. 

Combining sampling with machine learning can be relevant for many other settings. For instance, we have only considered ensemble learning by voting and averaging, but not the other approaches such as gating or stacking. Or, the training for deep learning or neural networks taking graph, instead of vector, as inputs may need to be based on subsets of data or subgraphs, due to limited memory size or computing power, for which sampling methods and design-based estimation require their own studies.

\bibliographystyle{biometrika}
\bibliography{paper-ref}

\end{document}